\documentclass{article} 

\usepackage{preprint,times}

\usepackage{amsmath,amsfonts,bm}

\def\eqref#1{equation~\ref{#1}}

\def\1{\bm{1}}

\DeclareMathAlphabet{\mathsfit}{\encodingdefault}{\sfdefault}{m}{sl}
\SetMathAlphabet{\mathsfit}{bold}{\encodingdefault}{\sfdefault}{bx}{n}

\usepackage{graphicx} 
\usepackage{booktabs}
\usepackage{hyperref}
\usepackage{url}

\usepackage{amsthm}
\theoremstyle{plain}
\newtheorem{theorem}{Theorem}
\newtheorem{lemma}{Lemma}
\newtheorem{corollary}{Corollary}
\theoremstyle{definition}
\newtheorem{definition}{Definition}
\theoremstyle{remark}
\newtheorem*{remark}{Remark}

\title{What Does Model Growth Really Add? \\ Functional Capacity Beyond Parameter Count}

\author{
    Dante Lok \\
    Votee AI \& Beever AI \\
    \texttt{dante.lok@votee.ai} \\
    \texttt{dante.lok@beever.ai}
}

\begin{document}

\maketitle

\begin{abstract}
Model growth adds parameters to a trained checkpoint and continues training, aiming to 
increase capacity while reusing the existing model instead of training from scratch. 
But does growing a trained model actually create new functional capacity that survives subsequent training? 
Parameter count alone cannot answer this question, and ``capacity'' is itself ambiguous;
it can refer to parameter count, effective dimensionality, or the ability to represent new functions.
We measure it directly, as the number of new, functionally independent directions a growth step adds,
formalized as the rank increment of the model's functional Jacobian and estimated matrix-free at scale. 
For zero-gated function-preserving growth, this increment is exactly $K$ under a transversality condition,
i.e. $\operatorname{rank}(J_{\mathrm{grown}})=\operatorname{rank}(J_{\mathrm{old}})+K$. 
Continually training a Transformer grown from $253\mathrm{M}$ to $857\mathrm{M}$ parameters across three domains, 
we find these added directions persist, with all $K=36$ remaining independent after continual learning,
even when the model catastrophically forgets (perplexity $28\to481$) or we deliberately destroy the old function ($25.9\to127$). 
Geometric capacity and retained function are therefore decoupled. The functional capacity introduced by growth persists,
while the previously learned function is not retained. A second, geometrically distinct mechanism, Net2Net,
reaches the same $K$-dimensional capacity from directions that are dormant at initialization ($0 \to K$) and then retains it,
so the decoupling is not specific to the zero-gate parameterization.
The implication is precise. Catastrophic forgetting need not reflect a loss of functional capacity,
and preserving dimensionality alone is insufficient to preserve the learned function.
\end{abstract}

\section{Introduction}

\textbf{Motivation.} ``Capacity'' means three different things in this literature. When a growth
method \citep{chen2016net2net, wang2024lemon, wang2023ligo, chen2022bert2bert} ``adds capacity
cheaply,'' capacity is \emph{parameters and FLOPs}, a claim about size that says nothing about what
survives training. When the plasticity literature reports that continual training ``loses
capacity'' \citep{dohare2024plasticity, neuroplastic2025}, capacity is \emph{effective rank}, a
collapsing ability to keep learning. When continual-learning theory reports that ``model capacity
deteriorates'' \citep{chakraborty2025clemc}, capacity is the \emph{ability to represent tasks under
the forgetting loss}. Three quantities, one word, and the conflation has a cost: it is unclear
whether a growth operator's advertised capacity persists or is capacity you pay for and lose.
Concretely, three things go unmeasured. First, what a growth operator actually adds: parameter
count only upper-bounds delivered capacity, yet the growth literature reports parameters and loss,
not the functional degrees of freedom a growth step contributes \citep{maddox2020effdim}. Second,
whether that capacity survives continual learning: prior work measures function-value capacity
(plasticity, retention), so whether the \emph{geometric} capacity growth adds persists, and how it
relates to what is retained, is unknown. Third, whether this depends on the growth mechanism or
holds for function-preserving growth in general. 

\textbf{What we measure.} We characterize the \emph{geometric} capacity of growth: the number
of new, functionally-independent directions a growth step adds, defined as the rank the
operator adds to the network's functional Jacobian. The functional Jacobian is a standard
analysis lens; prior work uses its rank to study, for instance,
trainability at initialization \citep{everett2026transformingrank}. We apply the same lens to
model growth: the functional capacity a \emph{growth operator} adds, and whether that capacity
persists through \emph{continual learning}. For zero-gated function-preserving (FP) growth, the operator adds exactly $K$ independent directions
under a transversality condition, i.e.\
$\operatorname{rank}(J_{\mathrm{grown}}) = \operatorname{rank}(J_{\mathrm{old}}) + K$ 
(Lemma~\ref{lem:fp} and Theorem~\ref{thm:t2}). We verify it with a
matrix-free estimator built on Jacobian--vector products, which recovers the rank without ever
materializing the Jacobian; a degeneracy stress test confirms the increment
is exact and non-trivial (zeroing a fraction of new blocks drops the rank by exactly that
fraction: $36 \to 27 \to 18 \to 9$). On a Transformer grown from $253$M to $857$M
parameters ($K = 36$ new blocks), this gives an added rank of exactly $36$.

\textbf{Capacity is decoupled from retained function.} We then continually train
the grown model across three domains (WikiText-103 $\to$ BookCorpus $\to$ code) and track the
added subspace. It is \emph{preserved}, with hard rank staying $36 \to 36$ and remaining transversal to
the old parameters, whether the old domain is catastrophically forgotten (perplexity
$28 \to 481$) or deliberately destroyed by replay onto a conflicting objective ($25.9 \to 127$). At
the same time, the subspace is \emph{re-tasked}: the directions specialized to one domain are
near-orthogonal to those specialized to the next. Geometric functional capacity is therefore
\textbf{decoupled} from retained function. For a residual branch with a linear output map, FP-at-init
depth growth admits two constructions: zero the scalar gate \citep{bachlechner2021rezero, yao2024msg, han2025loire} or zero the output weights
\citep{chen2016net2net, evci2022gradmax, wang2024lemon}. The experiments show the decoupling applies to both,
suggesting it is a property of function-preserving depth growth itself. Since the added
capacity is intact even in a model that has completely forgotten, forgetting is not a shortage of
capacity: keeping capacity high cannot, on its own, prevent it.

\textbf{Contributions.} Our contributions are a geometric analysis of model growth and its
persistence under continual learning, measured through functional-Jacobian rank.
\begin{enumerate}
\item \textbf{A decomposition of capacity.} We define the \emph{geometric} capacity of growth,
the number of new functionally-independent directions it adds (functional-Jacobian rank), and
separate it from the \emph{function-value} capacity, plasticity and retention, that the
continual-learning literature measures. This is what makes the decoupling below a measured result, 
not the trivial observation that rank and loss differ.
\item \textbf{The decoupling result.} The added $K$-dimensional capacity is
preserved ($36 \to 36$) under catastrophic forgetting and under deliberate function
destruction, while being re-tasked across domains, which is direct evidence that geometric
capacity and retained function are independent. The decoupling holds for both FP mechanisms
(gate-zero and Net2Net), showing its generality across function-preserving depth growth.
\item \textbf{A scalable measurement of functional capacity.} We compute functional-Jacobian
rank matrix-free using Jacobian--vector products, avoiding explicit Jacobian materialization.
We validate the measurement against exact rank and controlled degeneracy tests, making functional-capacity 
analysis practical at the scale of hundreds-of-millions-parameter Transformers.
\end{enumerate}

\section{Related Work}
\label{sec:related}

\subsection{Model growth and function-preserving initialization}
Model growth adds parameters to a trained checkpoint and continues training, reusing prior compute.
We organize growth methods by their function-preserving (FP) mechanism \citep{wei2016networkmorphism}:
\emph{how} the new component is made to contribute exactly zero at initialization so the grown model
matches the base. On the \textbf{depth} axis there are exactly two FP-at-init constructions.
\emph{Class~1 (zero the gate):} a scalar or structured gate on the new branch is initialized to zero,
annihilating the branch output regardless of its weights, as in MSG \citep{yao2024msg} and LOIRE
\citep{han2025loire}. Our operator, gate-zero, instantiates this mechanism, inspired by ReZero's
zero-initialized residual gate \citep{bachlechner2021rezero}, and is our Class-1 representative.
\emph{Class~2 (zero the output weights):} the outgoing weight matrix is initialized to zero
\citep{chen2016net2net, evci2022gradmax, wang2024lemon}; these differ only in the input-weight
initialization, not the FP mechanism, with Net2Net as representative. The same principle appears
outside growth: Fixup zero-initializes each residual branch \citep{zhang2019fixup}, ControlNet uses
zero-convolutions \citep{zhang2023controlnet}, and LoRA zero-initializes its low-rank update
\citep{hu2022lora}.

A third mechanism, \emph{Class~3 (canceling duplication)}, duplicates neurons or layers and rescales
downstream weights to preserve the function without zeroing \citep{chen2016net2net, chen2022bert2bert};
it is \textbf{width-native}, so we leave it to future work. We exclude methods that are \emph{not} FP
at initialization: LiGO \citep{wang2023ligo} learns a small-to-large linear map rather than a
closed-form FP construction, and stacking ($G_{\text{stack}}$) \citep{du2024stacking} and depth
up-scaling \citep{kim2023solar} stack layers without preserving the function, recovering the drop
through continued pretraining. Whether growth is even necessary for continual learning is contested:
a fixed model with replay can control forgetting \citep{ibrahim2024continual}, while FP growth has
been applied directly to lifelong learning \citep{han2025loire}. Growth's value therefore rests on
its adding real, durable functional capacity, which is what we measure.

\subsection{Rank, effective rank, and ``capacity''}
``Rank'' and ``capacity'' name several distinct quantities, which we separate along two axes:
the \emph{object} (which matrix's spectrum is measured) and the \emph{operator} (how the
spectrum is turned into a number). The operators are largely shared: \emph{effective rank} as
spectral entropy \citep{roy2007effrank}, which we borrow for our soft measure; \emph{effective
dimensionality} from the loss Hessian \citep{maddox2020effdim}; and \emph{intrinsic dimension}
via random projection \citep{li2018intrinsic}. What differs is the object. Prior work measures
the rank of \textbf{weights}, i.e.\ how far a model can be low-rank factorized while retaining
accuracy \citep{zerihun2025vit}; of \textbf{representations}, the hidden-state matrix, which
collapses toward rank one with depth in pure-attention networks \citep{dong2021attention,
huh2023lowrank}; or of the \textbf{loss Hessian / parameter curvature}, counting the parameter
directions the data constrains \citep{maddox2020effdim, li2018intrinsic, gurari2018tinysubspace}.
These established that parameter count overstates usable capacity, but only for \emph{fixed} models.

In contrast, Transforming Rank \citep{everett2026transformingrank} recasts skip connections
and normalization as mechanisms that preserve the branch / input--output Jacobian rank across depth, 
and shows the \emph{initialization} Jacobian rank predicts trainability. Instead of asking how a 
fixed architecture behaves at initialization, we measure what a \emph{growth operator adds} 
(the parameter Jacobian $\partial f/\partial\theta$) and how it evolves \emph{through continual learning}. 
This enables our analysis: the exact $+K$ increment growth adds, and its decoupling from retained function
under continual learning.

\subsection{Capacity in continual learning}
A parallel line uses ``capacity'' for a function-value quantity and finds it \textbf{deteriorates}
under continual training: representation effective rank collapses, costing plasticity
\citep{dohare2024plasticity}, which motivates growing or regenerating units to restore it
\citep{neuroplastic2025}, an idea with roots in task-wise network expansion
\citep{rusu2016progressive}; and model capacity defined through the forgetting loss is provably
non-stationary and degrades as tasks drift \citep{chakraborty2025clemc}. Our result separates the two
senses these works conflate under one word: the \emph{geometric} capacity a growth operator adds
(functional Jacobian rank) is \textbf{preserved} even as this function-value capacity is
lost; the two are decoupled.

\section{Functional Rank of Growth Operators}
\label{sec:theory}

\paragraph{Setup and notation.} Let $f_\theta : \mathcal{X} \to \mathbb{R}^{d_{\mathrm{out}}}$
be the model (for a language model, $d_{\mathrm{out}}$ is vocabulary size times sequence
positions on a probe input), with base parameters $\theta \in \mathbb{R}^{p}$ trained on domain
$A$. \emph{Gate-zero growth} inserts $K$ residual blocks, each of the form
$x \mapsto x + \alpha_i\, B_{\phi_i}(x)$ with scalar gate $\alpha_i$ and block weights
$\phi_i$, appended after the original blocks, with $\alpha = 0$ at initialization; the grown
parameters are $\theta' = (\theta, \phi, \alpha)$.

\begin{definition}[Functional Jacobian and its rank]
On a probe input $x$, the functional Jacobian is
$J_\theta(x) = \partial f_\theta(x)/\partial\theta \in \mathbb{R}^{d_{\mathrm{out}} \times p}$,
stacked over a probe set.
\begin{itemize}
\item The \emph{functional rank} is $\operatorname{rank} J$, measured
matrix-free via forward-mode Jacobian--vector products --- intuitively, how many independent
directions in parameter space the model can use to change its output.
\item Two measures from the singular values of $J$:
  \begin{itemize}
  \item \emph{soft rank} (Shannon rank): the effective rank \citep{roy2007effrank}
  $\operatorname{erank}(J) = \exp(H(\sigma))$, with $H$ the Shannon entropy of the
  $L_1$-normalized singular values $\sigma_i$ (Appendix~\ref{app:estimator}).
  \item \emph{hard rank} (numerical rank): $\#\{\sigma_i > \tau\,\sigma_{\max}\}$, the number of
  singular values above a relative tolerance $\tau$ (we use $\tau = 10^{-3}$).
  \end{itemize}
\end{itemize}
\end{definition}

All proofs are deferred to Appendix~\ref{app:proofs}.

\subsection{What growth adds}

\begin{lemma}[Function preservation at initialization, and the live direction]
\label{lem:fp}
For gate-zero growth with $\alpha = 0$:
\begin{itemize}
\item[(i)] $f_{\theta'}(x) = f_\theta(x)$ for every input $x$;
\item[(ii)] the new-weight block of the Jacobian vanishes, i.e.\
$\partial f_{\theta'}/\partial\phi \big|_{\alpha=0} = 0$;
\item[(iii)] the gate direction is live,
$\partial f/\partial\alpha_i = (\partial f/\partial\mathrm{out}_i)\, B_{\phi_i}(h_i(x))$, which is
generally nonzero since $B_{\phi_i}(h_i(x))$ is the output of block $i$ on its input activation.
\end{itemize}
Hence the functional directions opened by growth are the $K$ gate directions, and the rank
analysis is carried out over the $K$ gate columns.
\end{lemma}

\begin{theorem}[Rank additivity under transversality]
\label{thm:t2}
Let $J_{\mathrm{old}}$ have rank $r$, and let
$J_\alpha = [\,\partial f/\partial\alpha_1 \mid \cdots \mid \partial f/\partial\alpha_K\,]$ be
the $K$ gate columns.
\begin{itemize}
\item If the columns of $J_\alpha$ are linearly independent of
$\operatorname{col}(J_{\mathrm{old}})$ (\emph{transversality}), then
$\operatorname{rank}[\,J_{\mathrm{old}} \mid J_\alpha\,] = r + K$.
\item If transversality fails, the
rank increases by strictly less than $K$.
\end{itemize}
\end{theorem}

\begin{corollary}[Parameter count is not instantaneous functional dimension]
\label{cor:params}
At the gate-zero growth point, the first-order functional dimension added by growth satisfies:
\begin{itemize}
\item[(i)] $\Delta\operatorname{rank} \le K \le \min(P_{\mathrm{new}}, d_{\mathrm{out}})$, where
$P_{\mathrm{new}}$ is the number of new parameters, with equality $\Delta\operatorname{rank} = K$
under transversality (Theorem~\ref{thm:t2});
\item[(ii)] $K \ll P_{\mathrm{new}}$ whenever each gate
opens a block of many weights;
\item[(iii)] the remaining $P_{\mathrm{new}} - K$ new directions lie in the
Jacobian null space at the growth point (Lemma~\ref{lem:fp}(ii)).
\end{itemize}
This is a first-order
statement at initialization, not a claim that those directions are permanently redundant: they
become functional once the gates open. Parameter count therefore bounds from above, and in general
overstates, the instantaneous functional tangent dimension.
\end{corollary}

\subsection{What functional rank does not determine}

\begin{theorem}[Functional rank is a tangent dimension, not a position]
\label{thm:decoupling}
Let $F:\Theta\to\mathcal{F}$ be the parameter-to-function map, where $\mathcal{F}$ is the
finite-dimensional space of outputs on a fixed probe set, and let $F$ be $C^1$ on an open set
$U\subset\Theta$. Suppose $J(\theta) = DF_\theta$ has constant rank $r > 0$ on $U$. Then:
\begin{itemize}
\item[(i)] $F(U)$ is locally an $r$-dimensional embedded submanifold $M$ of $\mathcal{F}$, and for
every $\theta \in U$, $\operatorname{rank} J(\theta) = r = \dim T_{F(\theta)} M$, so the
functional rank is the dimension of the tangent space of $M$ at $F(\theta)$; 
\item[(ii)] $F$ is not
constant on $U$, so there exist $\theta_1, \theta_2 \in U$ with
$\operatorname{rank} J(\theta_1) = \operatorname{rank} J(\theta_2) = r$ and
$F(\theta_1) \ne F(\theta_2)$, and moreover, for any task loss $L_T$ nonconstant on $F(U)$, such a
pair can be chosen with $L_T(\theta_1) \ne L_T(\theta_2)$.
\end{itemize}
Consequently, functional rank determines the local dimension
of available functions, not which function is represented, and therefore does not determine
retention.
\end{theorem}

\begin{remark}
The loss gap can be unbounded: for the linear model $f_\theta(x) = \theta^{\top} x$, the Jacobian
$J = x^{\top}$ is identical for all $\theta$ (rank and column space fixed) while squared-error
loss grows without bound as $\theta$ scales. Same functional rank, different function, arbitrarily
different loss: rank fixes the tangent dimension, not the position on $M$.
\end{remark}

\section{Experiments}
\label{sec:experiments}

Our experiments measure whether growth adds the functional capacity the theory predicts, whether
that capacity survives continual learning, and how it relates to what the model retains.

\subsection{Setup}
\label{sec:exp-setup}

\textbf{Base model and growth.} The base is a $253$M-parameter, 12-layer decoder-only
Transformer trained on WikiText-103 (domain A); all growth experiments start from this single
shared checkpoint, so differences between grown models are attributable to the growth operator
and the subsequent training, not to the base. With growth factor $g = 4$, the model goes from 12
to 48 layers, adding $K = 36$ new blocks and reaching 857M parameters. Smaller factors ($g = 2$,
$g = 3$; $K = 12$, $24$; 454M, 656M) are used for the scale check in Table~\ref{tab:rq2}.

\textbf{Growth operators.} We grow by appending $K$ residual blocks after the original stack
(append layout), initialized by cloning the original blocks in round-robin order; this cloning
is what makes the added directions \emph{redundant} rather than random, a point we return to in
Section~\ref{sec:secondary}. A function-preserving (FP) operator makes each new block contribute
exactly zero at initialization, so the grown model reproduces the base model on every input. On
the depth axis, there are two ways to do this, and we implement one representative of each
(Section~\ref{sec:related}):
\begin{itemize}
\item \textbf{Gate-zero (Class 1, zero the gate).} Each new block is wrapped as
$x \mapsto x + \alpha_i\, B_{\phi_i}(x)$ with a scalar gate $\alpha_i$ initialized to zero. The
block weights $\phi_i$ are live but their contribution is annihilated by the gate; the gate
itself is the live direction at initialization (Lemma~\ref{lem:fp}).
\item \textbf{Net2Net (Class 2, zero the output weights).} Each new block's output projections
(the attention output projection and the MLP down projection) are initialized to zero, and the
gate is fixed at one. The block contributes zero through its weights rather than through a gate.
At initialization, its directions are therefore dormant (rank zero), and they activate only as
the output weights train.
\end{itemize}
For both mechanisms, we verify function preservation numerically at the moment of growth: the
maximum absolute difference in logits between base and grown model is $0.00$ (exact to float
precision) on a held-out batch.

\textbf{Clean function preservation.} Our production recipe for gate-zero growth includes three
stabilizing knobs: a distillation term toward the base model, a small gate warm-up, and clone
noise. To isolate the growth operator from these aids, all persistence and generality runs use
a \emph{clean-FP} configuration with all three disabled. The one exception is the
function-destruction counterexample (Section~\ref{sec:rq2}, E5), which deliberately uses the
replay recipe \emph{with} these knobs to match the original run that produced it.

\textbf{Continual-learning protocol.} Three domains are used in sequence: WikiText-103 (A, the base
domain), BookCorpus (B), and a code corpus (C). The two new domains are each capped at 118M tokens to
match the WikiText-103 base ($\approx 118$M tokens), so the three are size-balanced. We run a two-domain sequence
(A$\to$B) and a three-domain sequence (A$\to$B$\to$C). Each continual-learning phase trains for
10 epochs of 900 steps.
\begin{itemize}
\item \textbf{Isolation.} The original parameters are frozen and only the new blocks (and, for
gate-zero, their gates) are trained on the new domain. This is the regime for the persistence 
and generality results.
\item \textbf{Replay onto a conflicting objective.} The full model is trained on the new domain
with a replay buffer of old-domain data at fraction $\alpha$. Under our settings, this regime
\emph{destroys} the old-domain function (perplexity on A rises from 25.9 to 127) while the model
continues to learn; it is used only as the function-destruction counterexample (E5).
\end{itemize}

\textbf{Experiments.} We report eight experiments, referenced by identifier throughout
and mapped to specific rows and columns in each table.
\begin{itemize}
\item \textbf{What growth adds (Table~\ref{tab:rq1}).} \textbf{E1}: at the moment of growth
the added rank equals $K$ and is transversal to the old parameters. \textbf{E2}: a degeneracy
stress test that makes new blocks contribute zero and confirms the rank counts live functional
directions, not blocks.
\item \textbf{Does it survive continual learning (Table~\ref{tab:rq2}, Figure~\ref{fig:retasking}).}
\textbf{E3}: the added hard rank persists through continual learning. \textbf{E4}: it stays
transversal, rather than being absorbed into the old span. \textbf{E5}: it persists even when the
old-domain function is forgotten or deliberately destroyed, the decoupling counterexample.
\textbf{E6}: the preserved subspace is re-tasked across domains, not frozen.
\item \textbf{Generality test (Table~\ref{tab:rq3}, top).} \textbf{E7}: the same persistence
holds under a second, mechanistically opposite FP operator, Net2Net.
\item \textbf{Effective rank (Table~\ref{tab:rq3}, bottom).} \textbf{E8}: at fixed parameter count,
effective rank separates redundant from efficient growth, which hard rank and parameter count
cannot.
\end{itemize}

\subsection{Growth adds exactly $K$ functional directions}
\label{sec:rq1}

We first describe how a grown model is turned into a rank measurement, then verify
Theorem~\ref{thm:t2} at the moment of growth.

\paragraph{Measuring functional rank.}
\label{sec:measure}
On a probe input $x$, the functional Jacobian is
$J_\theta(x) = \partial f_\theta(x)/\partial\theta$. For the added directions, we take the $K$
gate columns $J_\alpha = [\partial f/\partial\alpha_1 \mid \cdots \mid \partial f/\partial\alpha_K]$
(Lemma~\ref{lem:fp}), stacked over a probe set. For Net2Net the gate is fixed at one, but
$\partial f/\partial\alpha_i$ still equals the block's output, zero at initialization and nonzero
after training (the dormant-to-active signature, Section~\ref{sec:rq3}).

\textbf{Matrix-free computation.} On one probe batch of $2 \times 128$ tokens over a $50{,}257$-way
vocabulary the output has $12.9$M dimensions, so at $857$M parameters the full Jacobian
($\approx 10^{16}$ entries) is far too large to form. We only need its $K$ gate columns and the
$64$ sampled old-direction columns of Section~\ref{sec:transversal}, each an output-sized vector
obtained by a single forward-mode Jacobian--vector product with a one-hot tangent; together they
occupy about $5$\,GB per probe batch, and a complete measurement takes a few minutes on one GPU.

\textbf{Two rank operators.} From the stacked columns we report two measures: the \textbf{hard
rank}, an integer count of singular values above $10^{-3}\,\sigma_{\max}$ (probe-independent, our
headline measure), and the soft \textbf{effective rank} \citep{roy2007effrank}, which depends on
the probe. Appendix~\ref{app:estimator} gives definitions and probe settings.

\textbf{Transversal residual.}\label{sec:transversal} Rank additivity (Theorem~\ref{thm:t2})
requires the added directions to be independent of the old parameter directions. We test this by
projecting the gate columns off the span of the old-parameter Jacobian: because that Jacobian is
far too large to form, we sample 64 random directions in the old parameter space, compute their
Jacobian columns by the same forward-mode procedure, orthonormalize them into a projector
$P_{\mathrm{old}}$, and report the hard rank of $(I - P_{\mathrm{old}})\,J_\alpha$. A residual rank
of $K$ means every added direction carries functional content outside the old span; a residual
rank below $K$ means some lie inside it. Appendix~\ref{app:estimator} gives the per-checkpoint
projector construction and the thresholding convention.

\textbf{Subspace rotation across domains.} We compare the gate subspace after training on domain B
with the subspace after subsequently training on domain C, via the principal angles between the two
$K$-dimensional column spaces, to see whether the preserved subspace is \emph{re-tasked} across
domains. A frozen subspace would show principal angles near zero; a preserved but re-oriented one
shows angles near the random baseline of two independent $K$-dimensional subspaces (near $90^\circ$).

\paragraph{Estimator validation.}
\label{sec:validation}
We validate the pipeline with an exact-rank self-test on a synthetic problem of known structure
(Appendix~\ref{app:estimator}); the degeneracy stress test at scale is experiment E2 below, and it
shows that the transversality condition of Theorem~\ref{thm:t2} is not a tautology at this scale.
Appendix~\ref{app:estimator} additionally reports a cautionary case in which a coarser proxy
manufactured a collapse that the functional-Jacobian estimator does not show.

\paragraph{Results.}
Theorem~\ref{thm:t2} predicts that under transversality growth adds exactly $K$ independent
directions. We verify this at the moment of growth.

The $K = 36$ gate columns have hard rank $36$ and are transversal to the old parameters (residual
$36$): growth adds exactly $K$ independent directions, each
carrying functional content outside what the base already spans (E1). The rank counts live
functional directions, not blocks: making a fraction of the new blocks degenerate by zeroing their
output projections drops the hard rank to exactly the number of live
blocks ($36/27/18/9$ for $0/25/50/75\%$ zeroed; Table~\ref{tab:rq1}), so the transversality
condition of Theorem~\ref{thm:t2} has real bite at this scale (E2).

\begin{table}[t]
\centering
\small
\caption{Growth adds exactly $K$ functional directions (gate-zero, $g=4$, $K=36$, at the
moment of growth).}
\label{tab:rq1}
\begin{tabular}{lcccc}
\toprule
New blocks made degenerate & $0\%$ (E1) & $25\%$ & $50\%$ & $75\%$ \\
\midrule
Live new blocks & 36 & 27 & 18 & 9 \\
Hard rank of added directions (E2) & 36 & 27 & 18 & 9 \\
\bottomrule
\end{tabular}
\end{table}

\subsection{The added capacity survives continual learning, and is decoupled from retained function}
\label{sec:rq2}

We track the added subspace through continual learning under conditions that progressively 
destroy the represented function (Table~\ref{tab:rq2}).

Across every condition, the dimension of the added $K$-dimensional subspace is preserved. Its hard
rank and its transversal residual both stay at $36 \to 36$ (grown $\to$ after CL), with zero
variance across seeds and probe batches, while the effective rank drifts only mildly and depends on
the probe (E3, E4). What moves is the function, which Table~\ref{tab:rq2} traces from merely
degraded, to catastrophically forgotten (domain B, perplexity $28 \to 481$), to deliberately
destroyed by replay onto a conflicting objective (domain A, $25.9 \to 127$), with the rank
unchanged at every step (E5). The destroyed case is decisive: the geometric capacity is
fully intact in a model whose learned function has been overwritten. The \emph{residual} stays at
$36$, not merely the rank, so the added directions remain independent of the old parameter span
rather than collapsing into it.

The subspace persists, but it is not frozen. The directions specialized to domain B and to domain
C are nearly orthogonal (median principal angle $\approx 86^\circ$, against a random baseline whose
angles all lie in $[85,90]^\circ$; about 5 of the $36$ directions are shared), so the dimension is
conserved while its orientation is repurposed from one domain to the next (E6;
Figure~\ref{fig:retasking}, with the full distribution in Appendix~\ref{app:subspace}).
Function-preserving growth thus adds a $K$-dimensional budget that continual
learning neither exhausts nor protects. Forgetting coincides with substantial reorientation of a
dimension-preserved subspace, not a loss of functional capacity.

\begin{figure}[t]
\centering
\includegraphics[width=0.72\linewidth]{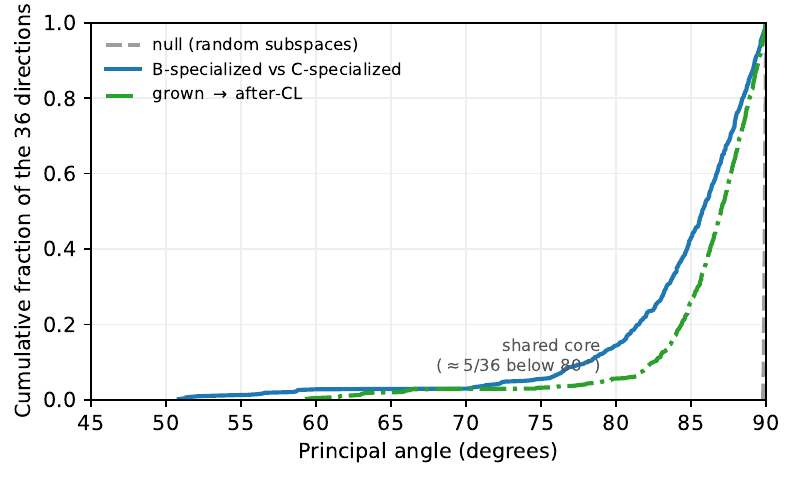}
\caption{\textbf{The added subspace is re-tasked, not frozen (E6).} Empirical CDF of the $K=36$
principal angles between gate subspaces, for domain B versus domain C and for grown$\to$after-CL.
Both track the near-$90^\circ$ random baseline rather than the near-$0^\circ$ a frozen subspace
would give, so the dimension rotates rather than staying fixed. Full distribution in
Appendix~\ref{app:subspace}.}
\label{fig:retasking}
\end{figure}

\begin{table}[t]
\centering
\small
\caption{Capacity survives continual learning and is decoupled from retained function
(gate-zero, clean-FP unless noted). The scale-check rows
report rank and residual only. Per-seed old-domain perplexities behind the ranges are in
Appendix~\ref{app:perseed}.}
\label{tab:rq2}
\begin{tabular}{llccl}
\toprule
Continual-learning condition & Seeds & Hard rank (E3) & Transv. residual (E4) & Old-domain function (E5) \\
\midrule
2-domain isolation (A$\to$B) & 42, 43, 44 & $36 \to 36$ & $36 \to 36$ & A degraded, PPL $25.9 \to 92$--$93$ \\
3-domain sequence (A$\to$B$\to$C) & 42, 43, 44 & $36 \to 36$ & $36 \to 36$ & B forgotten, PPL $28 \to 481$ \\
\textbf{Replay onto conflict} & 42, 43, 44 & $36 \to 36$ & $36 \to 36$ & A destroyed, PPL $25.9 \to 125$--$127$ \\
\midrule
\multicolumn{5}{l}{\emph{Scale check (2-domain, seed 42)}} \\
$g = 2$, $K = 12$ & 42 & $12 \to 12$ & $12 \to 12$ & -- \\
$g = 3$, $K = 24$ & 42 & $24 \to 24$ & $24 \to 24$ & -- \\
\bottomrule
\end{tabular}
\end{table}

\subsection{Generality of the property across FP-Growth Operator}
\label{sec:rq3}

\textbf{Why two operators suffice.} An FP-at-init depth construction must make a new residual
branch output zero. For a branch with a linear output map, this leaves exactly two constructions,
either zeroing the gate that scales the branch or zeroing the branch's output weights (Section~\ref{sec:related});
cancelling duplication, the remaining FP mechanism, is width-native and outside our depth scope.
Gate-zero and Net2Net are one representative of each, so within this class they exhaust the
FP-at-init depth constructions.

Table~\ref{tab:rq3} shows both mechanisms reaching the same endpoint by opposite routes.
Gate-zero's directions are live at initialization and \emph{preserved} (hard rank $36 \to 36$);
Net2Net's output weights are zeroed, so its directions are dormant at initialization and
\emph{reached} ($0 \to 36$), activating through training. After continual learning, both have hard
rank $36$ and transversal residual $36$ across all three seeds (E7). The residual is what makes
this non-trivial. A trained 48-layer network trivially has full rank in its last blocks, but
residual $36$ means those directions are independent of the old parameter span (and this holds as
we sample far more of that span; Appendix~\ref{app:estimator}), i.e.\ genuinely
\emph{added} capacity rather than a re-spanning of what the base already had. Net2Net's higher
after-CL effective rank ($33.4$--$33.6$ against gate-zero's $30.8$--$31.0$ on the same probe) is consistent
with its carrying no cloning redundancy. We keep the verb precise. Net2Net supports the survival
and generality claims, not the at-initialization $+K$ additivity of Section~\ref{sec:rq1}, which
is a live-gate result.

\paragraph{Effective rank separates redundant from efficient growth.}
\label{sec:secondary}
Effective rank can separate two operators that share the same $K$, a distinction hard rank cannot
make because it is fixed at $K$ for any transversal growth. At fixed $g = 4$ and identical
parameter count, clone-initialized new blocks have mean effective rank $28.8$ against $35.7$ for random
initialization (Table~\ref{tab:rq3}, lower panel), showing that cloning adds redundant directions
that hard rank and parameter count miss. That clones are redundant is intuitive; the point is that
the measure sees it (E8).

\begin{table}[t]
\centering
\small
\caption{Generality test and effective-rank separation of redundant from efficient growth
($g=4$, $K=36$, clean-FP). Effective rank uses the random-token probe;
Appendix~\ref{app:rank-vs-g} extends E8 to $g = 2, 3$.}
\label{tab:rq3}
\begin{tabular}{llcccc}
\toprule
Mechanism & Class & Seeds & Hard rank & Transversal residual & Eff. rank after CL \\
\midrule
\multicolumn{6}{l}{\emph{E7: persistence across the two FP mechanisms (grown $\to$ after CL)}} \\
Gate-zero (zero gate) & 1 & 42, 43, 44 & $36 \to 36$ & $36 \to 36$ & 30.9 / 30.8 / 31.0 \\
Net2Net (zero output weights) & 2 & 42, 43, 44 & $0 \to 36$ & $0 \to 36$ & 33.4 / 33.6 / 33.6 \\
\midrule
\multicolumn{6}{l}{\emph{E8: at the moment of growth, $g=4$, 857M parameters}} \\
New-block initialization & & Seed & Hard rank & & Eff. rank at growth \\
Clone (round-robin) & & 42, 43, 44 & 36 & & 28.9 / 28.7 / 28.8 \\
Random & & 42, 43, 44 & 36 & & 35.7 / 35.6 / 35.7 \\
\bottomrule
\end{tabular}
\end{table}

\section{Limitations}
\label{sec:limitations}

\textbf{Scope of the growth studied.} Our claims are about function-preserving depth growth. We
test one representative of each of the two FP-at-init constructions available to a linearly-gated depth branch, so
members within a class that differ only in input initialization (GradMax, LEMON-T1) are not
run, and two cases lie outside the paper: width growth (Class~3, canceling duplication) and growth
that is not function-preserving at initialization. The
generality result should therefore be read as holding across these two FP depth-growth
constructions, not across growth in general (Section~\ref{sec:future}).

\textbf{Scale.} All grown models descend from a single $253$M base checkpoint, grown to
at most 857M parameters under our compute budget. The scale check varies the growth factor, and
hence $K$, but not the base, so we have no evidence from a second, independently trained base.

\textbf{Continual-learning regimes.} We do not sweep continual-learning methods;
regularization-based methods such as elastic weight consolidation \citep{kirkpatrick2017ewc},
distillation toward the previous model \citep{li2018lwf}, other replay fractions, and unprotected
fine-tuning are untested. This bounds only the continual-learning results. Rank additivity at
growth is a property of the initialization, independent of the training method, and the mechanism
comparison uses one regime for both operators. How other continual-learning constraints re-task the preserved subspace remains open.

\section{Future Work}
\label{sec:future}

Our results are scoped to function-preserving depth growth, and the most natural extensions
follow from what that scope leaves out.

\textbf{The width axis and the third FP mechanism.} Function-preserving growth admits a third
mechanism, canceling duplication (Net2WiderNet \citep{chen2016net2net}, bert2BERT
\citep{chen2022bert2bert}, LEMON-T2 \citep{wang2024lemon}), which is width-native and so outside
this paper's depth scope. Extending the rank analysis to width growth would make Class~3
measurable, and would require re-deriving rank additivity for width, where the added directions
are not gate columns.

\textbf{Growth that is not function-preserving.} Whether persistence extends to non-FP growth is
open. Plain stacking of trained blocks ($G_{\text{stack}}$ \citep{du2024stacking}) and depth up-scaling
\citep{kim2023solar} start far from the
base function, and LiGO \citep{wang2023ligo} learns a small-to-large linear map rather than a
closed-form function-preserving initialization; neither is function-preserving by construction.

\textbf{Rank-aware growth.} Functional rank is a diagnostic today; making it a control signal first
requires showing it predicts downstream performance. Our redundancy result relates effective rank
to how growth was initialized (Section~\ref{sec:secondary}), not to what a grown model can
subsequently learn. Controlled runs that vary the added rank at fixed parameter count would test
whether rank \emph{causes} downstream differences, and whether a rank-aware policy could add
capacity efficiently rather than repair what was lost.

\section{Conclusion}
\label{sec:conclusion}

For zero-gated function-preserving growth, the \emph{geometric} capacity a growth operator adds,
the rank increment of its functional Jacobian, is exactly $K$ under a transversality condition, and
a matrix-free estimator recovers it at the scale of an $857$M-parameter Transformer without ever
forming the Jacobian.

Our main finding is that this geometric capacity is decoupled from retained function. Through
continual learning, the added $K$-dimensional subspace is preserved and stays transversal to the
old parameters, even when the old domain is catastrophically forgotten or its function is
deliberately destroyed, while its orientation is re-tasked from one domain to the next. The same
behavior holds for both function-preserving depth-growth constructions we study, a
zero-initialized gate and zeroed output weights. The consequence is precise: forgetting coincides
with reorientation of a dimension-preserved subspace rather than a shortage of functional
capacity, and keeping dimensionality high cannot on its own preserve the learned function.

\clearpage
\bibliography{references}
\bibliographystyle{preprint}

\appendix
\clearpage
\section{Proofs}
\label{app:proofs}

\subsection{Proof of Lemma~\ref{lem:fp}}
\begin{proof}
\leavevmode
\begin{itemize}
\item[(i)] Each new block contributes $\alpha_i B_{\phi_i}(\cdot)$ to its residual stream; at
$\alpha_i = 0$, the block is the identity, so the composed map equals the base map.
\item[(ii)] By the chain rule,
$\partial f/\partial\phi_i = (\partial f/\partial\mathrm{out}_i)\,\alpha_i\,
(\partial B_{\phi_i}/\partial\phi_i) = 0$ since $\alpha_i = 0$.
\item[(iii)] $\partial f/\partial\alpha_i = (\partial f/\partial\mathrm{out}_i)\, B_{\phi_i}(h_i(x))$ carries no
factor of $\alpha_i$ and is therefore generally nonzero.
\end{itemize}
\end{proof}

\subsection{Proof of Theorem~\ref{thm:t2}}
\begin{proof}
Write $J = [\,J_{\mathrm{old}} \mid J_\alpha\,]$ and let $P$ be the orthogonal projector onto
$\operatorname{col}(J_{\mathrm{old}})$. The column space of $J$ is the sum of
$\operatorname{col}(J_{\mathrm{old}})$ and $\operatorname{col}(J_\alpha)$, and its part outside
$\operatorname{col}(J_{\mathrm{old}})$ is spanned by the columns of $(I-P)J_\alpha$; hence
$\operatorname{rank} J = \operatorname{rank} J_{\mathrm{old}} + \operatorname{rank}\big((I-P)J_\alpha\big)$.
\emph{Transversality} means that no nonzero combination of the gate columns lies in
$\operatorname{col}(J_{\mathrm{old}})$, i.e.\ the gate columns are linearly independent of the old
column space; equivalently, $(I-P)J_\alpha$ has full column rank $K$.
\begin{itemize}
\item If transversality holds, $(I-P)J_\alpha$ has full column rank $K$, so
$\operatorname{rank} J = r + K$.
\item If transversality fails, some nonzero combination of the gate columns lies in
$\operatorname{col}(J_{\mathrm{old}})$, so $(I-P)J_\alpha$ has a nontrivial kernel and rank strictly
less than $K$, giving $\operatorname{rank} J < r + K$.
\end{itemize}
\end{proof}

\subsection{Proof of Corollary~\ref{cor:params}}
\begin{proof}
\leavevmode
\begin{itemize}
\item[(i)] At the growth point the only live new directions are the $K$ gate columns
(Lemma~\ref{lem:fp}(iii)) and the new-weight columns vanish (Lemma~\ref{lem:fp}(ii)), so the
added rank is at most $K$; $K \le d_{\mathrm{out}}$ since rank is bounded by the number of rows,
and $K \le P_{\mathrm{new}}$ since each gate is one of the new parameters. Equality
$\Delta\operatorname{rank} = K$ under transversality is Theorem~\ref{thm:t2}.
\item[(ii)] Each gate is a single parameter that scales a block $B_{\phi_i}$ of many weights, so
$P_{\mathrm{new}}$ counts the $K$ gates together with all block weights $\phi$; hence
$K \ll P_{\mathrm{new}}$ whenever each block carries many weights.
\item[(iii)] The remaining $P_{\mathrm{new}} - K$ new-weight directions have zero Jacobian columns
at the growth point by Lemma~\ref{lem:fp}(ii), so they lie in the Jacobian null space; this
vanishing carries the factor $\alpha_i$ and disappears once $\alpha_i \ne 0$, which is why the
statement is first-order at initialization only.
\end{itemize}
\end{proof}

\subsection{Proof of Theorem~\ref{thm:decoupling}}
\begin{proof}
\leavevmode
\begin{itemize}
\item[(i)] The constant-rank theorem: about any $\theta_0 \in U$ there are $C^1$ coordinates in
which $F(u) = (u_1, \dots, u_r, 0, \dots, 0)$, so $F(U)$ is locally an embedded
$r$-dimensional submanifold and $\operatorname{im} DF_\theta = T_{F(\theta)} M$ has dimension
$r$.
\item[(ii)] If $F$ were constant on the open set $U$, then $DF \equiv 0$ on $U$, contradicting
$r > 0$; so $F$ takes at least two values on $U$, at points that share rank $r$ by hypothesis. If
$L_T$ is non-constant on $F(U)$, choose $\theta_1, \theta_2 \in U$ with
$L_T(F(\theta_1)) \ne L_T(F(\theta_2))$; both have rank $r$, and $F(\theta_1) \ne F(\theta_2)$
since a non-constant function of $F$ separates them.
\end{itemize}
\end{proof}

\begin{remark}[When the hypotheses hold]
Theorem~\ref{thm:decoupling} is local and generic: same-rank, different-function pairs exist
wherever the rank is constant and positive. This holds on the open set where the rank attains its
maximum, exactly where our measured checkpoints sit ($\operatorname{rank} = K = 36$), and our
model, with smooth activations, is $C^1$ as the theorem requires.
\end{remark}

\clearpage
\section{Estimator details}
\label{app:estimator}

\textbf{Rank operators and probes.} From the singular values $\sigma$ of the stacked Jacobian
columns, we report two rank measures. The \emph{hard rank} at tolerance $\tau$ is
\[
  r_\tau(J) = \#\{\, i : \sigma_i > \tau\,\sigma_{\max} \,\},
  \qquad \tau = 10^{-3},
\]
an integer, probe-independent in practice, and our headline measure. The \emph{effective rank}
\citep{roy2007effrank} is
\[
  \operatorname{erank}(J) = \exp\big(H(\sigma)\big),
  \qquad H(\sigma) = -\sum_i \bar\sigma_i \log \bar\sigma_i,
  \qquad \bar\sigma_i = \frac{\sigma_i}{\sum_j \sigma_j},
\]
the exponential Shannon entropy of the $L_1$-normalized singular values: a soft, real-valued
measure that depends on the probe distribution, reported as secondary. Unless stated
otherwise, each measurement uses 3 probe batches of $2$ sequences of length 128, drawn as uniformly
random tokens or as held-out WikiText text; the hard rank agrees across the two probe types, and we
state which was used when reporting effective rank.

\textbf{Exact-rank self-test.} On a synthetic problem with known structure (ambient dimension 400,
$K = 6$ planted directions), the matrix-free pipeline recovers the exact planted rank, recovers
residual rank $K$ when the planted directions are orthogonal to a planted old span, and recovers
residual rank zero when they lie inside it, exercising both the hard-rank threshold and the
transversal-residual computation end to end.

\textbf{Per-checkpoint projector.} The old-parameter projector $P_{\mathrm{old}}$ used for the
transversal residual (Section~\ref{sec:transversal}) is recomputed from the current model at each
checkpoint; it is not a fixed reference. Under isolation, the old parameters are frozen, but the
downstream map from the old representation to the output changes as the appended blocks train, so
$\partial f/\partial\theta_{\mathrm{old}} =
(\partial f/\partial h_A)(\partial h_A/\partial\theta_{\mathrm{old}})$ changes through its first
factor even though the second is fixed.

\textbf{Thresholding.} Hard rank counts singular values above $\tau\,\sigma_{\max}$ with
$\tau = 10^{-3}$. For the residual $(I - P_{\mathrm{old}})J_\alpha$, we threshold relative to
$\sigma_{\max}$ of $J_\alpha$ itself, not of the residual, so that a vanishing residual (added
directions lying inside the old span) reads as rank zero rather than spuriously full.

\textbf{Coarse proxies can manufacture collapse.} An earlier coarse proxy, the effective rank of
the gradient covariance projected to a low dimension, reported an apparent rank collapse under
continual learning. The functional-Jacobian estimator does not reproduce it, and the drop did not
survive multi-seed replication: it was an artifact of the low-dimensional projection and a single
early checkpoint, not a functional-rank collapse. This is part of why we measure the functional
Jacobian directly rather than a gradient-covariance surrogate.

\textbf{Sampling the old span (transversality stress test).} The transversal residual projects
$J_\alpha$ off a subspace spanned by $m$ randomly sampled old-parameter directions, so independence
from $m = 64$ samples is a priori weaker than independence from the full old span. We therefore vary
the sampling budget $m$ and the probe length, which sets the output dimension
$D = 2 \times \mathrm{seq} \times 50{,}257$, and recompute the residual. Every grid point uses a
single probe batch of two sequences at seed $42$ in fp32.
Table~\ref{tab:transv-sweep} shows the outcome: the residual hard rank is $36 = K$ in every cell of
a grid over probe lengths $\mathrm{seq} \in \{16, 32, 64, 128\}$ and budgets
$m \in \{64, \dots, 1024\}$. Sampling more of the old span does move its closest direction nearer to
$J_\alpha$ (the minimum principal angle falls, for example $64.1^\circ \to 42.5^\circ$ as
$m: 64 \to 1024$ at $\mathrm{seq} = 16$), but never enough to remove a dimension: the added
directions stay \emph{independent} of the sampled old span at every budget, while not being
\emph{orthogonal} to it. The projector is QR throughout; the points that exceed a single GPU's
memory use a $D$-sharded tall-skinny QR across GPUs, validated to reproduce the single-GPU result
exactly (rank $36$; principal angles to $10^{-2}$ degrees) at $\mathrm{seq} = 128, m = 256$.

\begin{table}[h]
\centering
\small
\caption{Transversality stress test: minimum principal angle (degrees) between the $K = 36$ added
gate directions and the span of $m$ sampled old-parameter directions, across probe length
($\mathrm{seq}$) and sampling budget $m$. \textbf{The residual hard rank is $36 = K$ in every cell},
so independence from the sampled old span is invariant to how much of it we sample, even as the
minimum angle shrinks with $m$.}
\label{tab:transv-sweep}
\begin{tabular}{lccccc}
\toprule
$\mathrm{seq} \backslash m$ & $64$ & $128$ & $256$ & $512$ & $1024$ \\
\midrule
$16$  & $64.1$ & $57.2$ & $52.3$ & $47.3$ & $42.5$ \\
$32$  & $71.7$ & $67.4$ & $62.0$ & $56.4$ & $51.0$ \\
$64$  & $77.5$ & $73.0$ & $67.7$ & $62.7$ & $56.8$ \\
$128$ & $79.6$ & $74.8$ & $69.8$ & $64.8$ & $59.2$ \\
\bottomrule
\end{tabular}
\end{table}

\clearpage
\section{Per-seed continual-learning detail}
\label{app:perseed}

Table~\ref{tab:rq2} reports the old-domain function as a range over seeds $42/43/44$; the per-seed
values are in Table~\ref{tab:perseed-e5}. In every run, the added subspace is preserved (hard rank
and transversal residual $36 \to 36$) while the old-domain function degrades, is forgotten, or is
destroyed, so the decoupling holds seed by seed, not only on average. Perplexity is the old-domain
value before continual learning versus after it ($\mathrm{PPL}_A$ for the two-domain and replay
conditions; $\mathrm{PPL}_B$ after training on the third domain for the three-domain sequence).

\begin{table}[h]
\centering
\small
\caption{Per-seed detail behind Table~\ref{tab:rq2} (E5). Old-domain perplexity before $\to$ after
continual learning, with the added-subspace hard rank and transversal residual (grown $\to$
after-CL). Rank and residual are $36 = K$ in every run; only the retained function varies.}
\label{tab:perseed-e5}
\begin{tabular}{lccc}
\toprule
Seed & Old-domain PPL (pre $\to$ post) & Hard rank & Transv. residual \\
\midrule
\multicolumn{4}{l}{\emph{2-domain isolation A$\to$B, A degraded ($\mathrm{PPL}_A$)}} \\
$42$ & $25.9 \to 93.0$ & $36 \to 36$ & $36 \to 36$ \\
$43$ & $25.9 \to 92.2$ & $36 \to 36$ & $36 \to 36$ \\
$44$ & $25.9 \to 93.3$ & $36 \to 36$ & $36 \to 36$ \\
\midrule
\multicolumn{4}{l}{\emph{3-domain sequence A$\to$B$\to$C, B forgotten ($\mathrm{PPL}_B$)}} \\
$42$ & $28.3 \to 480.7$ & $36 \to 36$ & $36 \to 36$ \\
$43$ & $28.4 \to 459.1$ & $36 \to 36$ & $36 \to 36$ \\
$44$ & $28.4 \to 437.0$ & $36 \to 36$ & $36 \to 36$ \\
\midrule
\multicolumn{4}{l}{\emph{Replay onto conflict, A destroyed ($\mathrm{PPL}_A$)}} \\
$42$ & $25.9 \to 127.7$ & $36 \to 36$ & $36 \to 36$ \\
$43$ & $25.9 \to 127.4$ & $36 \to 36$ & $36 \to 36$ \\
$44$ & $25.9 \to 125.5$ & $36 \to 36$ & $36 \to 36$ \\
\bottomrule
\end{tabular}
\end{table}

\clearpage
\section{Subspace rotation of the added directions}
\label{app:subspace}

For the $K = 36$ new-block gate directions, we form the JVP columns
$J_\alpha = \partial f/\partial\alpha$ (Appendix~\ref{app:estimator}) and compute the principal
angles between two gate subspaces; the bin distribution (counts out of $36$) shows how many
directions are shared (small angle) versus re-tasked (near $90^\circ$). Angles are averaged over
5 probe batches, and for B-vs-C over seeds $42/43/44$ (agreement to std $\le 1.4$ per bin). The
random baseline is two independent $36$-dimensional subspaces in $\mathbb{R}^D$
($D = 2 \times 128 \times 50{,}257$; 5 draws), which in this dimension are essentially
orthogonal, with all $36$ angles in $[85,90]^\circ$. This is the same measurement as
Figure~\ref{fig:retasking}.

\begin{table}[h]
\centering
\small
\caption{Principal-angle distribution: mean count of the $36$ directions per angular bin
(degrees). The random baseline concentrates every direction in $[85,90]^\circ$; B-vs-C has a small
sub-$80^\circ$ core and is otherwise near-orthogonal.}
\label{tab:subspace}
\begin{tabular}{lcccccc}
\toprule
Comparison & $[0,45)$ & $[45,60)$ & $[60,70)$ & $[70,80)$ & $[80,85)$ & $[85,90]$ \\
\midrule
Random baseline & 0.0 & 0.0 & 0.0 & 0.0 & 0.0 & 36.0 \\
B-vs-C (mean over seeds) & 0.0 & 0.9 & 0.1 & 4.1 & 10.3 & 20.5 \\
Grown $\to$ after-CL (gate-zero) & 0.0 & 0.1 & 0.9 & 0.9 & 7.4 & 26.7 \\
\bottomrule
\end{tabular}
\end{table}

Against the random baseline, which places all $36$ angles at $\ge 85^\circ$, the B-vs-C distribution has about
$15/36$ below $85^\circ$ and about $5/36$ below $80^\circ$, but only about one direction below
$60^\circ$: a small shared core, with the rest rotated to near-orthogonal. The dimension is
conserved throughout (hard rank $36/36$; main text), so forgetting is capacity reuse rather than
loss.

\textbf{Caveat.} The B-specialized reference is a proxy: the two-domain (after-B) state was
checkpointed only at seed 42, so the fixed B reference is the exp1 seed-42 checkpoint, a
same-recipe but ${\sim}3\%$-off stand-in ($\mathrm{PPL}_B$ $29.2$ versus $28.3$); the C seed is
varied over $42/43/44$.

\clearpage
\section{Effective rank across growth factors}
\label{app:rank-vs-g}

Table~\ref{tab:rank-vs-g} extends the E8 comparison of Table~\ref{tab:rq3} from $g = 4$ to
$g \in \{2, 3, 4\}$ ($K = 12, 24, 36$). For each growth factor, the new blocks are either cloned
round-robin from the base (the gate-zero operator) or freshly randomly initialized; the gate is
zero in both cases, so the grown model is function-preserving. We measure the gate Jacobian at
the moment of growth with the random-token probe ($5$ probe batches of $2 \times 128$ tokens)
and repeat over seeds $42/43/44$, where the seed sets the clone noise and the random
initialization.

\begin{table}[h]
\centering
\small
\caption{Effective rank of the $K$ added directions at the moment of growth, by growth factor
and new-block initialization (mean $\pm$ std over seeds $42/43/44$). Hard rank equals $K$ in
every cell and for every seed.}
\label{tab:rank-vs-g}
\begin{tabular}{cclccc}
\toprule
$g$ & $K$ & New-block initialization & Hard rank & Eff. rank at growth & Eff. rank$/K$ \\
\midrule
$2$ & $12$ & Clone (round-robin) & $12$ & $11.39 \pm 0.02$ & $0.949$ \\
    &      & Random              & $12$ & $11.96 \pm 0.02$ & $0.996$ \\
\midrule
$3$ & $24$ & Clone (round-robin) & $24$ & $20.39 \pm 0.04$ & $0.850$ \\
    &      & Random              & $24$ & $23.81 \pm 0.04$ & $0.992$ \\
\midrule
$4$ & $36$ & Clone (round-robin) & $36$ & $28.81 \pm 0.05$ & $0.800$ \\
    &      & Random              & $36$ & $35.67 \pm 0.04$ & $0.991$ \\
\bottomrule
\end{tabular}
\end{table}

Hard rank and parameter count cannot distinguish the two initializations, since both add exactly
$K$ directions at every $g$. Effective rank can. Random initialization stays near full at every
growth factor (effective rank$/K \ge 0.99$), while cloning falls from $0.95$ at $g = 2$ to $0.80$
at $g = 4$, consistent with round-robin cloning reusing each base block $g - 1$ times. The effective
rank varies by at most $0.05$ across seeds.

\end{document}